\documentclass[11pt]{article}

\usepackage[utf8]{inputenc} 
\usepackage[T1]{fontenc}    
\usepackage[bookmarks=false]{hyperref}
\usepackage{url}            
\usepackage{booktabs}       
\usepackage{amsfonts}       
\usepackage{nicefrac}       
\usepackage{microtype}      
\usepackage{mathtools}
\usepackage{xcolor}
\usepackage[export]{adjustbox}
\usepackage{subcaption}
\usepackage{xspace}
\usepackage{bm}

\usepackage[T1]{fontenc}
\usepackage{lmodern}
\usepackage[ruled,vlined]{algorithm2e}
\usepackage{algorithmic}
\usepackage{amsthm}
\newtheorem{theorem}{Theorem}
\newtheorem{lemma}[theorem]{Lemma}
\usepackage{amsmath}
\usepackage{bbm}

\usepackage{graphicx}
\graphicspath{ {./images/} }
\usepackage{fullpage}
\usepackage{cleveref}

\theoremstyle{definition}
\newtheorem{definition}{Definition}[section]
\DeclareMathOperator*{\argmin}{arg\,min}
\newcommand{\pr}[1]{\text{Pr}\big[ #1 \big] }
\newcommand{\prob}[2]{\text{Pr}_{#1}\big[ #2 \big] }
\newcommand{\expect}{\mathbb{E} }

\newcommand{\pdf}{p}

\newcommand{\eps}{\epsilon}

\newcommand{\approxoracle}{\mathcal{O}_\alpha}
\newcommand{\approxoracleF}[1]{\approxoracle \Big( #1 \Big)}
\newcommand{\oracle}{\mathcal{O}}

\newcommand{\dataset}{\mathcal{D}}
\newcommand{\DS}{\dataset}
\newcommand{\inner}[2]{\big\langle #1, #2 \big\rangle}

\newcommand{\ObjDisc }{\textsf{OPDisc}}
\newcommand{\ObjSamp }{\textsf{OPSamp}}
\newcommand{\lossSet}{\mathcal{L}}
\newcommand{\Lone}[1]{||#1||_{1}}

\newcommand{\loss}{L}
\newcommand{\lossF}[1]{\loss \big(#1 \big)}
\newcommand{\lossregnormed}{\bar{\loss}}

\newcommand{\wopt}{\hat{w}}
\newcommand{\wcon}{v}
\newcommand{\bprojI}{\pi}
\newcommand{\bproj}[1]{\bprojI\big(#1\big)} 

\newcommand{\infdia}{D_{\infty}} 

\newcommand{\noiseset}[1]{\mathcal{E}\big(#1\big)}
\newcommand{\noisemap}[1]{g_{#1}}

\newcommand{\discspace}{\mathcal{W}_\tau}

\newcommand{\algo}{\mathcal{M}}

\newcommand{\cW}{\mathcal{W}}

\newcommand{\Lap}[1]{\mathrm{Lap}(#1)}
\newcommand{\Exp}[1]{\mathrm{Exp}(#1)}

\usepackage{color-edits}
\addauthor{sw}{blue}

\title{Oracle Efficient Private Non-Convex Optimization}
\date{}
\author{Seth Neel\thanks{Wharton Statistics Department, University of Pennsylvania. Email: \href{mailto:sethneel93@gmail.com}{sethneel93@gmail.com}} \and Aaron Roth\thanks{Department of Computer and Information Sciences, University of Pennsylvania.  This material is based upon work supported by the United States Air Force and DARPA under Contract No FA8750-16-C-0022. Any opinions, findings and conclusions or recommendations expressed in this material are those of the author(s) and do not necessarily reflect the views of the United States Air Force and DARPA. Email: \href{mailto:aaroth@cis.upenn.edu}{aaroth@cis.upenn.edu.}} \and Giuseppe Vietri\thanks{Department of Computer Science and Engineering, University of Minnesota. Supported by the GAANN fellowship from the U.S. Department of Education. Email: \href{mailto:vietr002@umn.edu}{vietr002@umn.edu}} \and Zhiwei Steven Wu\thanks{Department of Computer Science and Engineering, University of Minnesota. Supported in part by a Google Faculty Research Award, a J.P. Morgan Faculty Award, a Mozilla research grant, and a Facebook Research Award. Email: \href{mailto:zstevenwu@cmu.edu}{zstevenwu@cmu.edu}}}

\begin{document}
\maketitle

\begin{abstract}
One of the most effective algorithms for differentially private learning and optimization is \emph{objective perturbation}. This technique augments a given optimization problem (e.g. deriving from an ERM problem) with a random linear term, and then exactly solves it. However, to date, analyses of this approach crucially rely on the convexity and smoothness of the objective function, limiting its generality. We give two algorithms that extend this approach substantially. The first algorithm requires nothing except boundedness of the loss function, and operates over a discrete domain. Its privacy and accuracy guarantees hold even without assuming convexity. This gives an oracle-efficient optimization algorithm over arbitrary discrete domains that is comparable in its generality to the exponential mechanism. The second algorithm operates over a continuous domain and requires only that the loss function be bounded and Lipschitz in its continuous parameter. Its privacy analysis does not require convexity. Its accuracy analysis does require convexity, but does not require second order conditions like smoothness. Even without convexity, this algorithm can be generically used as an oracle-efficient optimization algorithm, with accuracy evaluated empirically. We complement our theoretical results with an empirical evaluation of the non-convex case, in which we use an integer program solver as our optimization oracle. We find that for the problem of learning linear classifiers, directly optimizing for 0/1 loss using our approach can out-perform the more standard approach of privately optimizing a convex-surrogate loss function on the Adult dataset.
\end{abstract} 

\thispagestyle{empty} \setcounter{page}{0}
\clearpage

\section{Introduction}
Consider the general problem of optimizing a function $\loss:\lossSet^n \times \mathcal{W}\mapsto \mathbb{R}$ defined with respect to a dataset $\DS \in \lossSet^n$ and a parameter $w \in \mathcal{W}$: $\arg\min_w \loss(\DS, w)$. This general class of problems is ubiquitous, and includes combinatorial optimization problems, empirical risk minimization problems, and synthetic data generation problems amongst others. We say that such a function $\loss$ is $1$-sensitive in the dataset $\DS$ if changing one datapoint in $\DS$ can change the value of $\loss(\DS, w)$ by at most 1, for any parameter value $w$. Suppose that we want to solve an optimization problem like this subject to the constraint of differential privacy. The \emph{exponential mechanism} provides a powerful, general-purpose, and often error-optimal method to solve this problem \cite{MT07}. It requires no assumptions on the function other than that it is $1$-sensitive (this is a minimal assumption for privacy: more generally, its guarantees are parameterized by the sensitivity of the function). It has indeed been used to solve private learning \cite{KLNRS08}, combinatorial optimization \cite{GLMRT10}, and synthetic data generation problems \cite{BLR08} subject to differential privacy, often optimally. Unfortunately, the exponential mechanism is generally infeasible to run: its implementation (and the implementation of related mechanisms, like ``Report-Noisy-Max'' \cite{DR14}) requires the ability to enumerate the parameter range $\mathcal{W}$, making it infeasible in most learning settings. When $\loss(\DS,w)$ is continuous, convex, and satisfies second order conditions like strong convexity or smoothness, the situation is better: there are a  number of algorithms available, including simple output perturbation \cite{objective1} and objective perturbation \cite{objective1,objective2,objective3}. This partly mirrors the situation in non-private data analysis, in which convex optimization problems can be solved quickly and efficiently, and most non-convex problems are NP-hard in the worst case.

In the non-private case, however, the worst-case complexity of optimization problems does not tell the whole story. For many non-convex optimization problems, such as integer programming, there are fast heuristics that not only reliably succeed in optimizing functions deriving from real inputs, but can also certify their own success. In such settings, can we leverage these heuristics to obtain practical private optimization algorithms? In this paper, we give two novel analyses of \emph{objective perturbation} algorithms that extend their applicability to 1-sensitive non-convex problems (and more generally, bounded sensitivity functions). We also get new results for \emph{convex} problems, without the need for second order conditions like smoothness or strong convexity. Our first algorithm operates over a discrete parameter space $\mathcal{W}$, and requires no further assumptions beyond 1-sensitivity for either its privacy or accuracy analysis --- i.e. it is comparable in generality to the exponential mechanism. The second algorithm operates over a continuous parameter space $\mathcal{W}$, and requires only that $\loss(\DS,w)$ be Lipschitz-continuous in its second argument. Its privacy analysis does not require convexity. Its accuracy analysis does --- but does not require any 2nd order conditions. We implement our first algorithm to directly optimize classification error over a discrete set of linear functions on the Adult dataset, and find that it substantially outperforms private logistic regression.

\subsection{Related work}
Objective perturbation was first introduced by \cite{objective1}, and analyzed for the special case of strongly convex functions. Its analysis was subsequently improved and generalized \cite{objective2,objective3} to apply to smooth convex functions, and to tolerate a small degree of error in the optimization procedure. Our paper is the first to give an analysis of objective perturbation without the assumption of convexity, and the first to give an accuracy analysis without making second order assumptions on the objective function even in the convex case. \cite{objective1} also introduced the related technique of \emph{output perturbation} which perturbs the exact optimizer of a strongly convex function.

The work most closely related to our first algorithm is \cite{neel2018use}, who also give a similar ``oracle efficient'' algorithm for non-convex differentially private optimization: i.e. reductions from non-private optimization to private optimization. Their algorithm (``Report Separator Perturbed Noisy Max'', or RSPM) relies on an implicit perturbation of the optimization objective by augmenting the dataset $\DS$ with a random collection of examples drawn from a \emph{separator set}. The algorithms which we introduce in this paper are substantially more general: because they directly perturb the objective, they do not rely on the existence of a small separator set for the class of functions in question. They also can yield improved accuracy bounds in cases where both techniques apply: see Sections \ref{sub:compare} and \ref{sec:experiment}.  \cite{neel2018use} also give a generic method to transform an algorithm (like ours) whose privacy analysis depends on the success of the optimization oracle, to an algorithm whose privacy analysis does not depend on this, whenever the optimization heuristic can certify its success (integer program solvers have this property). Their method applies to the algorithms we develop in this paper.  Our second algorithm crucially uses an $\ell_1$ stability result recently proven by \cite{suggala2019online} in the context of online learning.

\section{Preliminaries}\label{sec:prelims}
We first define a dataset, a loss function with respect to a dataset, and the two types of optimization oracles we will call upon. We then define differential privacy, and state basic properties.

A dataset $\DS \subset \lossSet^{n}$ is defined as a (multi)set of $G$-Lipschitz loss functions $l$. (Note that frequently, the dataset will explicitly contain ``data points'', and the loss functions will be implicitly defined). For $w$ in a parameter space $\cW \subset \mathbb{R}^{d}$, the loss on dataset $\dataset$ is defined to be
$$\lossF{\DS, w} = \sum_{l\in\DS} l(w)$$
We will define two types of perturbed loss functions, and the corresponding oracles which are assumed to be able to optimize each type. These will be used in our discrete objective perturbation algorithm in Section \ref{sec:discrete} and our sampling based objective perturbation algorithm in Section \ref{sec:sampling} respectively.

Given a vector $\eta \in \mathbb{R}^{d}$, we define the perturbed loss to be:
$$\bar L(\DS, w, \eta) = \frac{\loss(\DS,w) -\inner{\eta}{w}}{n}$$
where $n=|\DS| $ is the size of the dataset $\DS$.
This is simply the loss function augmented with a linear term.

Let $\pi$ be the normalization function formally defined in Section~\ref{sec:discrete}, which informally maps a $d$-dimensional vector with $l_2$ norm at most $D$ to
a unit vector in $\mathbb{R}^{d+1}$. Given a vector $\eta \in \mathbb{R}^{d+1}$ We define the perturbed normalized loss to be:
$$\lossregnormed_\pi(\DS, w, \eta) = \frac{\loss(\DS,w) -\inner{\eta}{\bproj{w}}}{n}$$
%
%
\begin{definition}[Approximate Linear Optimization Oracle]\label{approxlin}
Given as input a dataset $\DS \in \mathcal{L}^n$ and a $d$-dimensional vector $\eta$, an $\alpha$-approximate linear optimization oracle $\oracle_\alpha$ returns $w^*  = \oracle_\alpha(\DS, \eta) \in \mathcal{W}$ such that

$$\lossregnormed(\DS, w^*, \eta) \leq \inf_{w \in \mathcal{W}}\lossregnormed(\DS, w, \eta)  + \alpha$$
\end{definition}
When $\alpha = 0$ we say $\oracle$ is a linear optimization oracle.

\begin{definition}[Approximate Normalized Linear Optimization Oracle]\label{approxlinpi}
Given as input a dataset $\DS \in \mathcal{L}^n$ and a $(d+1)$-dimensional vector $\eta$, an $\alpha$-approximate normalized linear optimization oracle $\oracle_{\alpha, \pi}$ returns $w^* = \oracle_{\alpha, \pi}(\DS, \eta) \in \mathcal{W}$ such that
$$\lossregnormed_\pi(\DS, w^*, \eta) \leq \inf_{w \in \mathcal{W}}\lossregnormed_\pi(\DS, w, \eta)  + \alpha$$
\end{definition}
When $\alpha = 0$ we say $\oracle_{\pi}$ is a normalized linear optimization oracle. We remark that while it seems less natural to assume an oracle for the normalized perturbed loss which involves the non-linearity $\pi(w)$, in the supplement we show how we can linearize this term by introducing an auxiliary variable and introducing a convex constraint. This is ultimately how we implement this oracle in our experiments.

\begin{definition}
A randomized algorithm $\algo: \lossSet^n \rightarrow \mathcal{W}$ is an $(\alpha, \beta)$-minimizer for $\mathcal{W}$ if for every dataset $\DS \in \lossSet^n$, with probability $1-\beta$, it outputs $\algo(\DS)=w$ such that:
$$\frac{1}{n}\loss(\DS, w) \leq \inf_{w^*\in\mathcal{W}} \frac{1}{n}\loss(\DS, w^*) + \alpha$$
\end{definition}

Certain optimization routines will have guarantees only for discrete parameter spaces:

\begin{definition}[Discrete parameter spaces]  A $\tau$-separated discrete parameter space $\discspace \subseteq \mathbb{R}^d$ is a discrete set such that for any pair of distinct vectors $w_1, w_2 \in \discspace$ we have $\|w_1-w_2\|_2\geq\tau$.
\end{definition}
Finally we define differential privacy.

We call two \emph{data sets} $\DS, \DS' \in \lossSet^n$  \emph{neighbors} (written
as $\DS \sim \DS'$) if $\DS$ can be derived from $\DS'$ by replacing a single
loss function $l_i \in \DS'$ with some other element of
$\lossSet$.

\begin{definition}[Differential Privacy \cite{DMNS06,DKMMN06}]
\label{defintion:dp}
Fix $\eps,\delta \geq 0$. A randomized algorithm $A:\lossSet^*\rightarrow \mathcal{O}$ is $(\eps,\delta)$-differentially private (DP) if for every pair of neighboring data sets $\DS \sim \DS' \in \lossSet^*$, and for every event $\Omega \subseteq \mathcal{O}$:
$$\Pr[A(\DS) \in \Omega] \leq \exp(\eps)\Pr[A(\DS') \in \Omega] + \delta.$$
\end{definition}
The Laplace distribution centered at $0$ with scale $b$ is the
distribution with probability density function
$\Lap{z|b} = \frac{1}{2b}e^{-\frac{|z|}{b}}$. We also make use of the exponential distribution which has density function  $\Exp{z|b} = \frac{1}{b}e^{-\frac{z}{b}}$ if $z\geq 0$ and $\Exp{z|b} = 0$ otherwise.

\section{Objective perturbation over a discrete decision space}
\label{sec:discrete}
In this section we give an objective perturbation algorithm that is
$(\eps, \delta)$-differentially private for any non-convex Lipschitz
objective over a discrete decision space $\discspace$. We assume that
each $l\in \mathcal{L}$ is $G$-Lipschitz over $\discspace$
w.r.t. $\ell_2$ norm: that is for any $w, w'\in \discspace$,
$|l(w) - l(w')| \leq G \|w - w'\|_2$. Note that if $l$ takes values in
$[0, 1]$, then we know $l$ is also $1/\tau$-Lipschitz due to the
$\tau$-separation in $\discspace$.


\textbf{The Normalization Trick}. The key technical innovation in this section of the paper is the modification of the standard objective perturbation algorithm by introducing a normalization step: rather than minimizing the perturbed loss, we minimize the perturbed normalized loss. Let $D$ be a bound on the maximum $\ell_2$ norm of any vector in
$\discspace$. We will make use of a normalization onto the unit sphere in one higher dimension. The normalization function $\bprojI:\mathbb{R}^d\rightarrow \mathbb{R}^{d+1}$ is defined as:
$$\bproj{w} = \left(w_1, \ldots, w_d, D\sqrt{1-\|w\|_2^2/D^2}\right) \frac{1}{D}$$
Note that $\|\bproj{w}\|_2 = 1$ for all $w\in\discspace$,  and also that for any $w,w' \in\discspace$,
\begin{equation}
\label{projfact}
\| \bproj{w} -\bproj{w'}\|_2^2  \geq \frac{1}{D^2} \|w-w'\|_2^2,
\end{equation}
since $\| \bproj{w} -\bproj{w'}\|_2^2 = \frac{1}{D^2}(\|w-w'\|_2^2 +  D^2(\sqrt{1-\|w\|_2^2/D^2}-\sqrt{1-\|w'\|_2^2/D^2})^2) \geq \frac{1}{D^2} \|w-w'\|_2^2$. This shows that normalizing into the $(d+1)$-dimensional sphere  can't force points too much closer together than they start. The intuition behind the privacy proof is that the linear perturbation term provides stability; specifically we will argue that for any value of the noise $\eta$ than induces a particular minimizer $\hat{w}$ on a dataset $\DS$, there is a nearby value $\eta'$ that would induce $\hat{w}$ on any adjacent dataset $\DS'$. The argument proceeds by contradiction: suppose that there existed some $v \neq \hat{w}$ that was the minimizer on $\DS'$. Then since $\DS$ and $\DS'$ only differ in one data point, the difference between the normalized losses of $v$ and $\hat{w}$ on $\DS'$ can be broken into three terms: the difference between their scores on $\DS$ and the original perturbation term $\eta$, the difference between their scores on the two data points that differ between $\DS, \DS'$,
%
and the inner product between their normalized difference $\pi(\hat w)-\pi(v)$
with $\eta' - \eta$.
 The first term is positive by virtue of $\hat{w}$ being the minimizer on the original dataset $\DS$. The second term can be lower bounded using Lipschitzness of $\mathcal{L}$.
The third term is lower bounded using the fact that $\eta'-\eta$ is chosen to maximize the inner product $\inner{\eta'-\eta }{\pi(\hat w) - \pi(v)}$ by making the change in noise $\eta'-\eta$ move in the direction of $\pi(\hat w)$
We can only guarantee this has a greater inner product with $\hat{w}$ than $v$ if $\|\pi(\hat{w})\|_2 = \|\pi(v)\|_2$
, which is the rationale behind the normalization trick. Then the whole expression can be shown to be lower bounded by $0$, contradicting the fact that $v$ is the unique minimizer of the normalized loss on $\DS'$.

\begin{algorithm}[H]
\label{algorithm:objper}
\SetAlgoLined
\KwIn{$\dataset = \{ l_i \}_{i=1}^n$, oracle $\oracle_{\pi}$ over $\discspace$, privacy parameters $\epsilon$, $\delta$}
$\sigma \leftarrow  \frac{7 G D^2 \sqrt{\ln{1/\delta}}}{\tau \epsilon}$ \;
Draw random vector $\eta \sim \mathcal{N} \big(0, \sigma^2 \big)^{d+1}$ and use the projected oracle to solve:
\[
  \hat w = \oracle_{\pi}(\DS, \eta) \in \argmin_{w\in \discspace}\lossregnormed_\pi(\DS, \eta, w)
\]

\KwOut{$\hat w$}
\caption{Objective Perturbation over Discrete Space \ObjDisc}
\end{algorithm}
%
%
%
\paragraph{Privacy Analysis}
We now  prove that $\ObjDisc$ is differentially private, illustrating the importance of the normalization trick. We then state an accuracy bound, which follows from a simple tail bound on the random linear perturbation term.
\begin{theorem}\label{dproof}
Algorithm \ref{algorithm:objper} is $(\epsilon, \delta)$-differentially private.
\end{theorem}

\begin{proof}
For any realized noise vector $\eta$, we write $\wopt = \oracle_\pi(\DS, \eta)$ as the output. Now consider the set of mappings $\mathcal{G}: W_\tau \times \mathbb{R}^{d+1}\rightarrow\mathbb{R}^{d+1}$. If we can show:

\begin{itemize}
\item $\exists \noisemap{} \in \mathcal{G}$ s.t. $\wopt = \oracle_\pi(\DS',\noisemap{}(\hat{w}, \eta))$ (Lemma~\ref{lemma:indicator})
\item $\pr{\eta} \approx \pr{\noisemap{}(\hat{w}, \eta)}$ (Lemma~\ref{lem:gauss_ratio})
\item W.p.$1, \argmin_{w \in W_\tau}\bar{\mathcal{L}}(\mathcal{D}, w, \eta)$ is unique, (Lemma~\ref{lemma:uniquew})
\end{itemize}
then the probability of outputting any particular $w$ on input $\DS$  is close to the corresponding probability, on input $\DS'$ as desired. Lemma~\ref{lem:gauss_ratio} follows from simple properties of the Gaussian distribution, and Lemma~\ref{lemma:uniquew} from discreteness of $W_\tau$, which are established in the Appendix. We focus on proving Lemma~\ref{lemma:indicator}, which is the central part of the proof.

\begin{lemma}
\label{lemma:uniquew}
Fix any $\tau$-separated vector space $\discspace$. For every dataset $\DS$ there is a subset $B\subset \mathbb{R}^{d+1}$ such that $\pr{\eta \in B} = 0$ and for any $\eta \in \mathbb{R}^{d+1}\setminus B$:
$$ \exists \text{ a unique minimizer } \wopt \in \argmin_{w\in \discspace} \lossF{\DS,w}-\inner{\eta}{\bproj{w}}$$
\end{lemma}

Denote  the set of of noise vectors that induce output $w$ on dataset $\DS$ by $\noiseset{\DS, w} =\{\eta : \oracle_\pi(\DS, \eta)=w  \}$.
Define our mapping $\noisemap{} \in \mathcal{G}$ by:
$$\noisemap{}(\hat{w}, \eta) \overset{\text{def}}{=} \noisemap{\wopt}(\eta) = \eta + \frac{2}{\tau } G D^2 \bproj{\wopt}  $$
Note that the vector $\eta'-\eta = \noisemap{\hat{w}}(\eta)-\eta$ is parallel to $ \pi(\hat{w})$ . 
Lemma~\ref{lem:gauss_ratio} shows that with high probability over the draw of $\eta$, $\pr{\eta} \approx \pr{\noisemap{\wopt}(\eta)}$.
\begin{lemma}
\label{lem:gauss_ratio}
Let $\eta \sim \mathcal{N}(0, \sigma^2)^{d+1}$, $ \sigma \gets \frac{7G^2D^2\sqrt{\log(1/\delta)}}{\tau \epsilon}$, and $w \in \mathcal{W}_\tau$. Then there exists a set $C \subset \mathbb{R}^{d+1}$ such that $\pr{\eta \in C^{c}} \geq 1-\delta$, and for all $r \in C^{c}$ if $p$ denotes the probability density function of $\eta$:
$$
\frac{\pdf(r)}{\pdf(\noisemap{w}(r))} \leq e^{\epsilon}
$$
\end{lemma}

\begin{lemma}
\label{lemma:indicator}
Fix any $\wopt$ and any pair of neighboring datasets $\DS,\DS'$. Let $\eta \in \noiseset{\DS, \wopt}$ be such that $\wopt$ is the unique minimizer $\wopt \in \inf_w \loss(\DS, w)-\inner{\eta}{\bproj{w}}$. Then $\noisemap{\wopt}(\eta)\in \noiseset{\DS', \wopt}$. Hence:
$$\mathbb{I}\{ \eta \in \noiseset{\DS, \wopt}\} \leq \mathbb{I}\{ \noisemap{\wopt}(\eta)\in \noiseset{\DS', \wopt}\}$$
\end{lemma}
\begin{proof}
Let $c=\frac{4}{\tau } G D^2$. Suppose that $\wcon \neq \hat{w}$ is the output on neighboring dataset $\DS'$ when the noise vector is $\noisemap{\wopt}(\eta)$. We will derive a contradiction. Since $v$ is the unique minimizer on $\DS'$:
\begin{align*}
\Big(\lossF{\dataset', \wcon}-\inner{\noisemap{\wopt}(\eta)}{\bproj{\wcon}}\Big)
 - \Big(\lossF{\dataset', \wopt}-\inner{\noisemap{\wopt}(\eta)}{ \bproj{\wopt}} \Big)
&<0
\end{align*}
Let $i$ be the index where $\DS$ and $\DS'$ are different, such that $l_i \in \DS$ and $l_i^{'} \in \DS'$. Then $\lossF{\DS', w} = \lossF{\DS, w} - l_i(w) + l_i^{'}(w)$. Now, write the loss function in terms of $\dataset$ and rearranging terms:
\begin{align*}
\bigg[ \Big(\lossF{\dataset, \wcon}-\inner{\eta}{ \bproj{\wcon}}\Big) - \Big(\lossF{\dataset, \wopt}-\inner{\eta}{ \bproj{\wopt}} \Big)\bigg]
+ \big(l_i(\wopt)-l_i(\wcon)\big) - \big(l_i^{'}(\wopt)-l_i^{'}(\wcon)\big)\\
+\inner{c\bproj{\wopt}}{ \bproj{\wopt}} - \inner{c\bproj{\wopt}}{ \bproj{\wcon}} < 0
\end{align*}

Since $\wopt$ is a unique minimizer for $\DS$ and $\eta$ then term in the square bracket is positive. Hence:
\begin{align*}
 \big(l_i(\wopt)-l_i(\wcon)\big) - \big(l_i^{'}(\wopt)-l_i^{'}(\wcon)\big)
 +\inner{c\bproj{\wopt}}{ \bproj{\wopt} - \bproj{\wcon}} < 0
\end{align*}

Since $l_i, l_i'$  are $G$-Lipschitz functions $\big(l_i(\wopt)-l_i(\wcon)\big) - \big(l_i^{'}(\wopt)-l_i^{'}(\wcon)\big) \geq -2G\|\wopt-\wcon\|_2$.  Now comes the importance of the normalization trick: because  $||\pi(v)||_2 = ||\pi(\hat{w})||_2 = 1$, $\inner{c\bproj{\wopt}}{ \bproj{\wopt}-\bproj{\wcon}} = \frac{c}{2}\|\bproj{\wopt}-\bproj{\wcon}\|_2^2$, by expanding  $\|\bproj{\wopt}-\bproj{\wcon}\|_2^2$. Note that without the normalization, this last term could be negative, breaking the contradiction argument.
Substituting this becomes:
\begin{align*}
-2G\|\wopt-\wcon\|_2 +\frac{c}{2}\|\bproj{\wopt}-\bproj{\wcon}\|_2^2 < 0
\end{align*}
For the next step we use inequality~\eqref{projfact}. We also apply the assumption
that for two vectors $\hat w  \neq v $ the following inequality holds $\|\wopt-\wcon\|_2\geq \tau$.

%
\begin{align*}
\frac{c}{2D^2} \|\wopt-\wcon\|_2^2 &< 2G\|\wopt-\wcon\|_2 && (\text{Inequality~\eqref{projfact}}) \\
\frac{c}{2D^2}\|\wopt-\wcon\|_2 &< 2G && (\text{Divide both sides by } \|\wopt-\wcon\|_2)\\
c\|\wopt-\wcon\|_2 &< 4 G D^2 \\
c\tau &< 4 G D^2 && (\text{By assumption }\|\wopt-\wcon\|_2\geq \tau)\\
c &< \frac{4 G D^2 }{\tau } && (\text{Divide both sides by } \tau)
\end{align*}
%
This contradicts $c = \frac{4 G D^2 }{\tau }$.
\end{proof}

Putting the Lemmas together:
\begin{align}
\notag
\pr{\oracle_\pi(\DS,\eta) \in S}
&=  \pr{\eta \in \bigcup_{\wopt} \noiseset{\DS, \wopt}} \\
\notag
&=\int_{\mathbb{R}^{d+1}} \pdf(\eta) \mathbb{I}\{ \eta \in \bigcup_{\wopt} \noiseset{\DS, \wopt}\}d\eta\\
%
\label{eq:line}
&=\int_{(\mathbb{R}^{d+1}\setminus B)\setminus C} \pdf(\eta) \mathbb{I}\{ \eta\in \bigcup_{\wopt} \noiseset{\DS, \wopt}\}d\eta +\int_{C} \pdf(\eta) \mathbb{I}\{ \eta\in\bigcup_{\wopt} \noiseset{\DS, \wopt}\}d\eta\\
\label{eq:lineI}
&\leq\int_{(\mathbb{R}^{d+1}\setminus C)\setminus B} \pdf(\eta) \mathbb{I}\{ \eta \in\bigcup_{\wopt} \noiseset{\DS, \wopt}\}d\eta + \delta  \\
\notag
&=\sum_{\wopt \in S} \int_{\mathbb{R}^{d+1}\setminus (C \cup B)} \pdf(\eta) \mathbb{I}\{ \eta \in \noiseset{\DS, \wopt}\}d\eta
+ \delta\\
\label{eq:lineII}
&\leq \sum_{\wopt \in S}  \int_{\mathbb{R}^{d+1}\setminus (C \cup B)} \pdf(\eta) \mathbb{I}\{ \noisemap{\wopt}(\eta) \in \noiseset{\DS', \wopt}\}d\eta + \delta \\
%
\label{eq:lineIII}
&\leq \sum_{\wopt \in S} \int_{\mathbb{R}^{d+1}\setminus (C \cup B)}
e^\epsilon\pdf(\noisemap{\wopt}(\eta)) \mathbb{I}\{ \noisemap{\wopt}(\eta) \in \noiseset{\DS', \wopt}\}d\eta   + \delta \\
%
\label{eq:lineIV}
&= \sum_{\wopt \in S}  \int_{\mathbb{R}^{d+1}\setminus (\noisemap{\wopt}(C) \cup \noisemap{\wopt}(B))}
e^\epsilon \pdf(\eta)\mathbb{I}\{\eta \in \noiseset{\DS', \wopt}\} \bigg|\tfrac{\partial \noisemap{\wopt}}{\partial \eta} \bigg| d\eta   \\
\notag
&\leq e^\epsilon \sum_{\wopt \in S}  \int_{\mathbb{R}^{d+1}}  \pdf(\eta)\mathbb{I}\{\eta \in \noiseset{\DS', \wopt}\} d\eta + \delta \quad  \\
\notag
&= e^\epsilon  \pr{\eta \in \bigcup_{\wopt} \noiseset{\DS', \wopt}} \\
\notag
&= e^\epsilon \pr{\oracle_\pi(\DS',\eta) \in S} + \delta
\end{align}
where equality \eqref{eq:line} follows from
Lemma~\ref{lemma:uniquew}. Then inequality \eqref{eq:lineI} holds because
  $C$ is chosen such that $\pr{\eta \in C} < \delta$.
 The inequality \eqref{eq:lineII} is from lemma \ref{lemma:indicator}
 and inequality \eqref{eq:lineIII} is from the bounded ration lemma \ref{lem:gauss_ratio}. Lastly, equality \eqref{eq:lineIV} follows
 because the mapping
 $\eta \rightarrow \noisemap{\wopt}(\eta)$ is one-to-one. Also note that
 $\bigg|\frac{\partial \noisemap{\wopt}}{\partial \eta} \bigg| = 1$
This completes the proof.
\end{proof} 
%
%
We now state the accuracy guarantee, which follows from a standard Gaussian tail bound. Then in Subsection~\ref{sub:compare} we compare this guarantee to the accuracy guarantee for the competing RSPM method for learning discrete hyperplanes, in order to shed some light on the accuracy guarantee in practice.
\begin{theorem}[Utility]\label{theorem:objperUtility}
Algorithm \ref{algorithm:objper} is an ($\alpha,\beta$)-minimizer for $\discspace^*$ with
$$\alpha  = \frac{14 G D^2\sqrt{2(d+1)\ln{(4/\beta)}\ln{(1/\delta)}}}{n\tau \epsilon}$$
\end{theorem}

\subsection{Comparing $\ObjDisc$ and RSPM}\label{sub:compare}
While both \ObjDisc{} and the RSPM algorithm of \cite{neel2018use}
require discrete parameter spaces, $\ObjDisc$ is substantially more
general in that it only requires the loss functions be Lipschitz,
whereas RSPM assumes the loss functions are bounded in $\{0, 1\}$ (and
hence $1/\tau$ Lipschitz over $\discspace$) and assumes the existence
of a small separator set (defined in the supplement). Nevertheless, we
might hope that in addition to greater generality, $\ObjDisc$ has
comparable or superior accuracy for natural classes of learning
problems. We show this is indeed the case for the fundamental task of
privately learning discrete hyperplanes, where it is better by a
linear factor in the dimension. We define the RSPM algorithm, for
which we must define the notion of a separator set, in the supplement.

\begin{theorem}[RSPM Utility \cite{neel2018use}]
\label{theorem:rspmUtility}
Let $\discspace^*$ be a discrete parameter space with a separator set of size $m$. The Gaussian RSPM algorithm is an oracle-efficient $(\alpha, \beta)$-minimizer for $\discspace^*$ for:
$$\alpha = O\bigg(\frac{m\sqrt{m \ln(2m/\beta) \ln(1/\delta)}}{\epsilon n}\bigg)$$
\end{theorem}
Let  $I_\tau$ be a $\tau$ discretization of $[-1, 1]^d$, e.g. $I_\tau = [-1, -1 + \tau, \ldots 0, \tau, 2\tau, \ldots 1]^{d}$. Let $W_\tau$ be the subset of vectors in this discretization that lie within the unit Euclidean ball: $W_\tau = I_\tau \cap S(1)^{d}$. $W_\tau$ is $\tau$-separated since any two distinct $w, w'$ differ in at least one coordinate by at least $\tau$. Moreover $W_\tau$ admits a separator set of size $m = \frac{2(d-1)}{\tau}$ (see the Appendix of \cite{neel2018use}. Since the loss functions $l_i(w) = \textbf{1}\{w \cdot x_i  \geq 1 \} \in \{0, 1 \}$ and $W_\tau$ is $\tau$-separated, the loss functions $l_i$ are $\frac{1}{\tau}$-Lipschitz.
By Theorem~\ref{theorem:rspmUtility}, RSPM has accuracy bound:
$$\alpha_{\text{RSPM}} = O\left(\frac{d \sqrt{d \log(d/\beta \tau)\log(1/\delta)}}{\tau \sqrt{\tau} \epsilon n}\right)$$
By Theorem~\ref{theorem:objperUtility} $\ObjDisc$ has accuracy bound:
$$
\alpha_{\ObjDisc} = O\left(\frac{\sqrt{d\log(1/\beta)\log(1/\delta)}}{n\tau^2 \epsilon}\right)
$$


Thus, in this case, \ObjDisc{} has an accuracy bound that is different by
a factor of roughly $d\sqrt{\tau}$. However, the bound of \ObjDisc~ is better only when $\tau$ is greater than $1/d^2$, pressing the question of how to set this parameter. The trade-off is that setting $\tau$ too large makes the algorithm \ObjDisc{} add too much noise to the objective, and our accuracy guarantee degrades very fast. On the other hand, if $\tau$ is too large, then we can miss the optimal solution to a large extent.
However, for practical scenarios, setting the value of $\tau$ to be much larger than $\tfrac{1}{d^2}$ gives a discretized decision space such that the optimal answer is not too far from the optimal on the corresponding continuous decision space. For instance, in our experiments, we set $\tau$ equals to one.

\section{Objective perturbation for lipschitz functions}

\newcommand{\OP}{\eta_1,\ldots, \eta_m}
\newcommand{\OracleAve}{A_{\oracle}}
\newcommand{\Oout}{M}

\label{sec:sampling}
We now present an objective perturbation algorithm (paired with an additional output perturbation step), which applies to arbitrary parameter spaces. The privacy guarantee holds for (possibly non-convex) Lipschitz loss functions, while the accuracy guarantee applies only if the loss functions are convex and bounded. Even in the convex case, this is a substantially more general statement than was previously known for objective perturbation: we don't require any second order conditions like strong convexity or smoothness (or even differentiability). Our guarantees also hold with access only to an $\alpha$-approximate optimization oracle.
We present the full algorithm in Algorithm~\ref{algorithm:sampling}.
It 1) uses the approximate linear oracle (in
Definition~\ref{approxlin}) to solve polynomially many perturbed optimization objectives, each with an independent random perturbation, and 2) perturbs the average of these
solutions with Laplace noise.

Before we proceed to our analysis, let us first introduce some
relevant parameters. Let $\mathcal{W}$ have $\ell_\infty$ diameter
$\infdia$, and $\ell_2$ diameter $D_2$. We assume that the loss
functions $l_i \in \mathcal{L}$ are $G$-Lipschitz with respect to
$\ell_1$ norm, and assume the loss functions are scaled to take values
in $[0,1]$. Our utility analysis requires convexity in the loss
functions, and essentially follows from the high-probability bounds on
the linear perturbation terms in the first stage and the output
perturbation in the second stage.
The privacy analysis of this algorithm crucially depends on a
stability lemma proven by \cite{suggala2019online} in the context of
online learning, and does not require convexity.\footnote{Compared to the bound in \cite{suggala2019online}, our bound has an additional factor of 2
since our neighboring relationship in Definition \ref{defintion:dp} is
defined via replacement whereas in \cite{suggala2019online} the
stability is defined in terms of adding another loss function.}

\begin{algorithm}[h]
\SetAlgoLined
\KwIn{Approximate optimization oracle $\mathcal{O}_\alpha$, a dataset $\DS = \{ l_i \}_{i=1}^n$, privacy parameters $\eps$, $\delta$.}
%
$\gamma \longleftarrow \frac{\sqrt{\eps}}{\sqrt{n}} d^{5/4} \sqrt{D_2}$\;
$m \longleftarrow \frac{\ln{(2d/\delta)}}{2\gamma^2}$\;
\For{ $k = 1$ \text{to} $m$}{
	$\sigma \longleftarrow \sqrt{\frac{ D_2\sqrt{2d}\epsilon}{ 250 G^2 d^2 \infdia ^2 (1 + \log(2/\beta))n}}$ \;
    Sample a random vector $\eta_k \sim \Exp{\sigma}^d$\;
    $$w_k \longleftarrow \approxoracleF{\dataset,\eta_k}$$
}
$\lambda \longleftarrow 4\infdia\gamma + 250 \sigma G d^2 \infdia ^2  + \frac{\alpha}{10G}$\;
$\mu \sim  \Lap{\lambda/\epsilon}^d$\;
\KwOut{$\frac{1}{m}\sum_{k=1}^m w_k + \mu$}
\caption{\ObjSamp}
\label{algorithm:sampling}
\end{algorithm}
%
\begin{theorem}[Utility]\label{sampleacc}
Assuming the loss functions are convex, Algorithm \ref{algorithm:sampling} is an $(\alpha',\beta)$-minimizer for $\frac{1}{n}\lossSet(w, \DS)$ with
%
%
%
\[
\alpha' = O\left(  \frac{d^{5/4} G D_\infty  \sqrt{D_2 \log(1/\beta)}}{\sqrt{\eps n}} + \frac{\alpha \log(1/\beta)}{\eps}\right)
\]
where $\alpha$ is the approximation error of the oracle
$\mathcal{O}_\alpha$.
\end{theorem}

\begin{proof}
For $\mu_i \sim \Lap{\frac{\lambda}{\epsilon}}, |\mu_i|\sim \Exp{\frac{\lambda}{\epsilon}}$.
 By Theorem $5.1$ in \cite{svante} which gives upper tail bounds for the sum of independent exponential random variables, we can conclude that $||\mu||_1 \leq r =  (1 + \log(2/\beta))\frac{\lambda}{\epsilon}$ with probability $1-\beta/2$.

Then by $G$-Lipschitzness with respect to the $l_1$ norm, with probability $1-\beta/2$:
$$\frac{1}{n}\lossSet \left(\frac{1}{m}\sum_{k=1}^m w_k + \mu, \DS\right) \leq \frac{1}{n}\lossSet \left(\frac{1}{m}\sum_{k=1}^m w_k, \DS \right) + Gr$$
We now focus on $\frac{1}{n}\lossSet(\frac{1}{m}\sum_{k=1}^m w_k, \DS)$. By the convexity of the loss functions, we have:
$$\frac{1}{n}\lossSet\left(\frac{1}{m}\sum_{k=1}^m w_k, \DS\right) \leq \frac{1}{m}\sum_{k=1}^m \frac{1}{n}\lossSet\left(w_k, \DS\right)$$
%
Since each $\frac{1}{n}\lossSet(w_k, \DS)$ is bounded in $[0,1]$ (since each $l_i \in [0,1]$) and independent, by Hoeffding's Inequality (see Appendix) with probability $1-\beta/2$:
$$
\left|\frac{1}{m}\sum_{k=1}^m \frac{1}{n}\lossSet(w_k, \DS)- \mathbb{E}_{\eta}\left[\frac{1}{n}\lossSet\left(w_{\DS, \eta}, \DS\right)\right]\right| \leq \sqrt{\frac{\log(4/\beta)}{2m}}
 $$
 So it suffices to show that $\mathbb{E}_{\eta}[\frac{1}{n}\lossSet(w_{\DS, \eta}, \DS)] - \argmin_{w \in \mathcal{W}}\frac{1}{n}\lossSet(w, \DS)$ is small.
 Fix $\tilde{w} =  \argmin_{w \in \mathcal{W}}\frac{1}{n}\lossSet(w, \DS)$ and $\eta\in \mathbb{R}^d$.
Let $w_{\DS, \eta} = \approxoracleF{\DS, \eta}$ be the oracle's output when the dataset is $\DS$ and the realized vector $\eta$, where $\eta\sim \Exp{\sigma}^d$ is sampled from the exponential distribution.
 Now by definition of $ \approxoracle$, for $w_{\DS, \eta} \leftarrow  \approxoracleF{ \DS , \eta }$, we have
 \begin{align*}
 	\frac{1}{n}\lossSet(w_{\DS, \eta}, \DS) - \frac{1}{n}\langle w_{\DS, \eta}, \eta \rangle \leq \frac{1}{n}\lossSet(\tilde{w}, \DS) - \frac{1}{n}\langle \tilde{w}, \eta \rangle + \alpha,
 \end{align*}
 hence
 \begin{align*}
 	\frac{1}{n}\lossSet(w_{\DS, \eta}, \DS)-\frac{1}{n}\lossSet(\tilde{w}, \DS) \leq \frac{1}{n}\langle w_{\DS, \eta} - \tilde{w}, \eta \rangle + \alpha
 \end{align*}
Then by Cauchy-Schwarz inequality we get
$\langle w_{\DS, \eta}-\tilde{w}, \eta \rangle \leq \|w_{\DS, \eta}-\tilde{w}\|_2 \|\eta\|_2 \leq D_2 \|\eta\|_2$. Taking expectation with respect to the random variable $\eta$ we get the following bound:

\begin{align*}
	\mathbb{E}_{\eta}\left[\frac{1}{n}\lossSet\left(w_{\DS, \eta}, \DS\right)\right] - \argmin_{w \in \mathcal{W}}\frac{1}{n}\lossSet(w, \DS) \leq \frac{1}{n}D_2 \mathbb{E}_\eta\left[\|\eta\|_2\right]
\end{align*}
 Now by Jensen's inequality, $\mathbb{E}_\eta[\|\eta\|_2] \leq \sqrt{\mathbb{E}_\eta\left[\|\eta\|^2_2\right]} =
\frac{\sqrt{2d}}{{\sigma}}$, where the last equality is by the variance of the exponential distribution.
Putting it all together, with probability $1-\beta$:
\begin{align*}
\frac{1}{n}\lossSet\left(\frac{1}{m}\sum_{k=1}^m w_k + \mu, \DS\right) - \argmin_{w \in W}\frac{1}{n}\lossSet\left(w + \mu, \DS\right) \leq Gr + \gamma\sqrt{\log(4/\beta)/2} + \alpha + \frac{1}{n}D_2\frac{\sqrt{2d}}{\sigma},
\end{align*}
Plugging in the value of $r$, $\lambda$ and expanding we get the following long expression:
\begin{equation}\begin{split}\label{equation:etavalue}
&G(1 + \log(2/\beta))\frac{\lambda}{\epsilon} + \gamma\sqrt{\log(4/\beta)/2} + \alpha + \frac{1}{n}D_2\frac{\sqrt{2d}}{\sigma} \\
&=G(1 + \log(2/\beta))\frac{\Big( 4\infdia\gamma + 250\sigma G d^2 \infdia ^2  + \frac{\alpha}{10G} \Big)}{\epsilon} + \gamma\sqrt{\log(4/\beta)/2} + \alpha + \frac{1}{n}D_2\frac{\sqrt{2d}}{\sigma} \\
&=\gamma \bigg( \frac{4G\infdia(1 + \log(2/\beta))}{\epsilon} \bigg)+ 
 \sigma \bigg( \frac{ 250 G^2 d^2 \infdia ^2 (1 + \log(2/\beta))}{\epsilon} \bigg)+ 
 \alpha \bigg( \frac{(1 + \log(2/\beta))}{10\epsilon} \bigg)+ \\ 
 & \quad \quad
 \gamma \Big( \sqrt{\log(4/\beta)/2} \Big)+ 
 \frac{1}{\sigma}\bigg( \frac{1}{n}D_2\sqrt{2d} \bigg)+  
 \alpha \\
&=\gamma A +
 \sigma B +
 \alpha C +
 \gamma D +
 \frac{E}{\sigma} +
 \alpha \quad \quad (\text{Setting placeholders } A,B,C,D,E)\\
&=\gamma (A + D)+
 \sqrt{BE} +
 \alpha (C+1) \quad \quad (\sigma = \sqrt{\frac{E}{B}})
\end{split}\end{equation}
The last step of equation \ref{equation:etavalue} comes from replacing in the value of $\sigma = \sqrt{\frac{ D_2\sqrt{2d}\epsilon}{ 250 G^2 d^2 \infdia ^2 (1 + \log(2/\beta))n}} = \sqrt{\frac{E}{B}}$. Replacing back the values of $A,B,C,D,E$ results in:
\begin{equation*}\begin{split}
%
& =
\gamma \bigg(\frac{G(1 + \log(2/\beta))4\infdia}{\epsilon} + \sqrt{\log(4/\beta)/2} \bigg)+
 \sqrt{ \frac{ 250 G^2 d^2 \infdia ^2 D_2\sqrt{2d} (1 + \log(2/\beta))  }{\epsilon n}   } + \\
 & \quad \quad\alpha \bigg( \frac{(1 + \log(2/\beta))}{10\epsilon}+1 \bigg)
\end{split}\end{equation*}
Finally, note that by the choice of the parameter $\gamma$, the first term has order at most that of the second term, which gives our stated bound.
\end{proof}
%
%

\paragraph{Privacy analysis} Before we prove that \cref{algorithm:sampling} satisfies differential-private in \cref{theorem:samplingdp}, we give some useful lemmas.
In this section, we use the following notation: Let $\OP$ be a sequence of of $m$ i.i.d $d$-dimensional noise vectors and $\OracleAve(\DS, \OP) = \frac{1}{m}  \sum_{i=1}^m \approxoracle(\DS,\eta_i)$ is the average output of $m$ calls to an $\alpha$-approximate oracle.
\begin{lemma}\label{lemma:stabilitylemma}
[Stability lemma \cite{suggala2019online}]
For any pair of neighboring data sets $\DS, \DS'$. Let $\approxoracle(\DS,\eta)$ and $\approxoracle(\DS',\eta)$ be the output of an approximate oracle on datasets $\DS$ and $\DS'$ respectively, where $\eta$ is a ramdom variable sampled from the exponential distribution with parameter $\sigma$. Then,
\[
\mathbb{E}_\eta\big[ \Lone{\approxoracle(\DS,\eta)-\approxoracle(\DS',\eta)} \big] \leq  250 \sigma G d^2 \infdia ^2  + \frac{\alpha}{10G}
\]
\end{lemma}
The next lemma gives a concentration bound on the output of the optimization oracle.
\begin{lemma}
\label{lemma:close_to_mean_lemma}
If $m=\frac{\ln{(2d/\delta)}}{2\gamma^2}$ and $\OP$ are $m$ random objective perturbation terms, then for $0\leq \gamma\leq 1$, with probability $1-\delta/2$:
\[
\|\OracleAve(\DS, \OP) - \mathbb{E}_\eta[\approxoracle(\DS, \eta)]\|_1 \leq 2\infdia \gamma
\]
where $\OracleAve(\DS, \OP) = \frac{1}{m}  \sum_{i=1}^m \approxoracle(\DS,\eta_i)$ and the randomness is taken over the sampling of $\OP$.
\end{lemma}
%
%
The next lemma combines Lemma~\ref{lemma:stabilitylemma} and Lemma~\ref{lemma:close_to_mean_lemma} to get high probability sensitivity
bound for the average output of the approximate oracle.
%
\begin{lemma}[High Probability $\ell_1$-sensitivity]\label{lemma:sensitivity}
Let $\OP$ be  $m= \tfrac{\ln{(2d/\delta)}}{\gamma^2}$ samples from the exponential distribution with parameter $\sigma$.
For any pair of neighboring datasets $\DS, \DS'$,
with probability $1-\delta$ over the random draws of $\OP$, we have
\begin{equation}\Lone{\OracleAve(\DS, \OP) - \OracleAve(\DS', \OP)} \leq 4D_\infty\gamma + 250 \sigma G d^2 D_\infty ^2  + \frac{\alpha}{10G}\label{ohnice}
\end{equation}
where $\OracleAve(\DS, \OP) =  \frac{1}{m}  \sum_{i=1}^m \approxoracle(\DS,\eta_i)$ and $\OracleAve(\DS', \OP) =  \frac{1}{m}  \sum_{i=1}^m \approxoracle(\DS',\eta_i)$.
\end{lemma}
\begin{proof}

To simplify notation let $X = \OracleAve(\DS, \OP)$ and $X' = \OracleAve(\DS', \OP)$. Note that $X$ and $X'$ are random variables and
the expectation $\mathbb{E}[X]$ is over the randomness of the optimization oracle $\oracle$: That is, $\mathbb{E}[X] = \mathbb{E}_\eta[\approxoracle(\DS, \eta)]$.
By Lemma \ref{lemma:close_to_mean_lemma}, if we run the approximate oracle $m=\frac{\ln{2d/\delta}}{2\gamma^2}$ times on each neighboring dataset $\DS, \DS'$, then by union bound we get that with probability $1-\delta$:
\begin{align}\label{eq:oracleconcentration}
\|X - \mathbb{E}_{\eta}[X]\|_1 \leq 2D_\infty\gamma \quad \text{ and } \quad
\|X' - \mathbb{E}_{\eta}[X']\|_1 \leq 2D_\infty\gamma
\end{align}
Adding both inequalities and applying the triangle inequality and linearity of expectations, we have
\begin{align*}
	\|X - X'\|_1 &=  \|X - X' -\expect[X] + \expect[X] - \expect[X'] + \expect[X']\|_1 \\
	&\leq  \|X - \expect[X]\|_1 + \|X' - \expect[X'] \|_1 + \|\expect[X] - \expect[X']\|_1   && (\text{triangle inequality}) \\
	&\leq  4 D_\infty\gamma + \|\expect[X] - \expect[X']\|_1  && (\text{\cref{eq:oracleconcentration}}) \\
	&\leq  4 D_\infty\gamma + \expect[\| X - X'\|_1] && (\text{Linearity of } \expect[\cdot])  \\
	&\leq  4 D_\infty\gamma +
250 \sigma G d^2 D_\infty ^2  + \frac{\alpha}{10G} && (\text{lemma \ref{lemma:stabilitylemma}})
\end{align*}
\end{proof}
%
%
\begin{theorem}\label{theorem:samplingdp}
Algorithm \ref{algorithm:sampling} is $(\epsilon, \delta)$-differentially private.
\end{theorem}
\begin{proof}[Proof sketch]
  Given a pair of neighboring data sets $\DS,\DS'$, we will condition
  on the set of noise vectors $\OP$ satisfy the
  $\ell_1$-sensitivity bound \eqref{ohnice}, which occurs with
  probability at least $1 - \delta$. Then the privacy guarantee
  follows from the use of Laplace mechanism.
\end{proof}

\begin{proof}
Fix two neighboring dataset $\DS, \DS'$ and any event $S\subset \mathbb{R}^d$. Let $\Oout(\DS)$ be a random variable denoting \cref{algorithm:sampling}'s output on dataset $\DS$. We want to show that $\pr{\Oout(\DS) \in S} \leq \exp(\varepsilon) \pr{\Oout(\DS')\in S} + \delta$.

Let us introduce some notation. We denote by $\OracleAve(\DS, \OP)$ the average of $m$ runs of $\approxoracle$ with dataset $\DS$ and sequence of i.i.d noise vectors $\OP $, sampled i.i.d from the exponential distribution. Hence, we can write the output of $\Oout(\cdot)$ as
as $\Oout(\DS) = \OracleAve(\DS, \OP) + \mu$, where $\OP$ is the objective perturbation term and $\mu$ is the output perturbation term.
Note that $\Oout(\DS)$ is a random variable that depends on the random perturbation and the output perturbation. If we fix the objective perturbation term $\OP$, we can write the output ad $\Oout(\DS, \OP)$.

Following Lemma \ref{lemma:stabilitylemma}, we let $\lambda \leftarrow 4D_\infty\gamma + 250 \sigma G d^2 D_\infty ^2  + \frac{\alpha}{10G}$ and define the event $B$ as
$$B = \{(\OP) \in \mathbb{R}^{(m,d)} : \| \OracleAve(\DS, \OP) - \OracleAve(\DS', \OP) \|_1 \leq \lambda \}$$
where $\lambda$ is the $\ell_1$-norm sensitivity bound from Lemma \ref{lemma:sensitivity}. Then, by the same lemma, if $\OP$ is drawn independently from the exponential distribution then
$\pr{B}\geq 1-\beta$.

Now we are ready for the main argument. Consider the joint probability $\pr{\Oout(\DS) \in S \land  B}$ and write it as:
\begin{equation}
\begin{split}\label{eq:conditional}
\pr{\Oout(\DS) \in S\land  B}
= \pr{\Oout(\DS) \in S|  B}\pr{ B} \\
\end{split}
\end{equation}

For the next part of the proof, we let $\pdf(\cdot)$ be the exponential distribution's joint probability density functions. We will upper bound the conditional probability $\pr{\Oout(\DS) \in S|  B}$. First, note that if we fix $(\OP)\in B$, then we can write the probability of $\Oout(\DS, \OP)\in S$ as
\begin{align}
\label{eq:conditionaloutputperturbation}
\pr{\Oout(\DS, \OP) \in S} = \prob{\mu\sim \Lap{\lambda/\eps}^d}{\OracleAve(\DS, \OP) + \mu \in S}
\end{align}
Furthermore, by Lemma \ref{lemma:sensitivity}, we have that $\|\OracleAve(\DS, \OP) - \OracleAve(\DS', \OP)\|_1 \leq \lambda$ for any $(\OP)\in B$. Therefore, conditioned on event $B$ the function $\OracleAve(\cdot, \OP)$ has bounced $\ell_1$ sensitivity, and since the Laplace mechanism is $\eps$-differentially private it follows by definition that
\begin{align}\label{eq:laplce}
\prob{\mu\sim \Lap{\lambda/\eps}}{\OracleAve(\DS, \OP) + \mu \in S|B}
\leq \exp(\eps)\prob{\mu\sim \Lap{\lambda/\eps}}{\OracleAve(\DS', \OP) + \mu \in S|B}
\end{align}

Putting the last two inequalities together, we can upper bound
$\pr{\Oout(\DS) \in S|  B}$ by

%
\begin{align}
\pr{\Oout(\DS) \in S|  B}
\notag
&= \int_{Z\in B} \pdf(Z)\pr{\Oout(\DS, Z) \in S}dZ \\
\notag
& = \int_{Z\in B} \pdf(Z)\prob{\mu\sim \Lap{\lambda/\eps}^d}{\OracleAve(\DS, Z)+\mu \in S}dZ && (\cref{eq:conditionaloutputperturbation})\\
\notag
& \leq \int_{Z\in B} \pdf(Z)\exp(\eps)\prob{\mu\sim \Lap{\lambda/\eps}^d}{\OracleAve(\DS', Z)+\mu \in S}dZ && (\cref{eq:laplce})\\
\label{eq:condupperbound}
&= \exp{(\epsilon)} \pr{\Oout(\DS') \in S| B}
\end{align}
Let $B^c$ be the complement of event $B$.
To complete the proof we use \cref{eq:condupperbound} and the fact that $\pr{B^c}\leq \delta$.
\begin{align*}
\label{eq:jointprob}
  \pr{\Oout(\DS) \in S}
  &= \pr{\Oout(\DS) \in S \land B} + \pr{\Oout(\DS) \in S \land  B^c} \\
  &\leq \pr{\Oout(\DS) \in S \land  B} + \pr{ B^c} \\
	&\leq \pr{\Oout(\DS) \in S\land  B} + \delta \\
	&= \pr{\Oout(\DS) \in S|  B}\pr{ B} + \delta &&( \cref{eq:conditional}) \\
	&\leq \exp(\eps)\pr{\Oout(\DS') \in S|  B}\pr{ B} + \delta &&( \cref{eq:condupperbound}) \\
	&= \exp(\eps)\pr{\Oout(\DS') \in S \land  B} + \delta  \\
	&\leq \exp(\eps)\pr{\Oout(\DS') \in S} + \delta  \\
\end{align*}

Therefore, $\pr{\Oout(\DS) \in S} \leq \exp{(\epsilon)} \pr{\Oout(\DS') \in S} + \delta$

\end{proof}

\section{Experiments}
\label{sec:experiment}
For our experiments, we consider the problem of privately learning a
the linear threshold function to solve a binary classification task.
Given a labeled data set $\{(x_i,y_i)\}_{i=1}^n$ where each
$x_i \in \mathbb{R}^d$ and $y_i \in \{-1,1\}$, the classification
problem is to find a hyperplane that best separates the positive from
the negative samples. A common approach is to optimize a convex
surrogate loss function that approximates the classification loss. We
use this approach (private logistic regression) as our baseline. In
comparison, using our algorithm $\ObjDisc$, we instead try and
directly optimize $0/1$ classification error over a discrete parameter
space, using an integer program solver. Although this can be
computationally expensive, we find that it is feasible for relatively
small datasets (we use a balanced subset of the Adult dataset with
roughly $n=15,000$ and $d = 23$ features, after one-hot encodings of
categorical features). In this setting, we find that $\ObjDisc$ can
substantially outperform private logistic regression. We remark that
``small data'' is the regime in which applying differential privacy is
most challenging, and we view our approach as a promising way forward
in this important setting.

\begin{figure}[h!]
\label{fig:adultI}
\centering
  \begin{subfigure}[b]{0.46\textwidth}
   	\includegraphics[width=\textwidth]{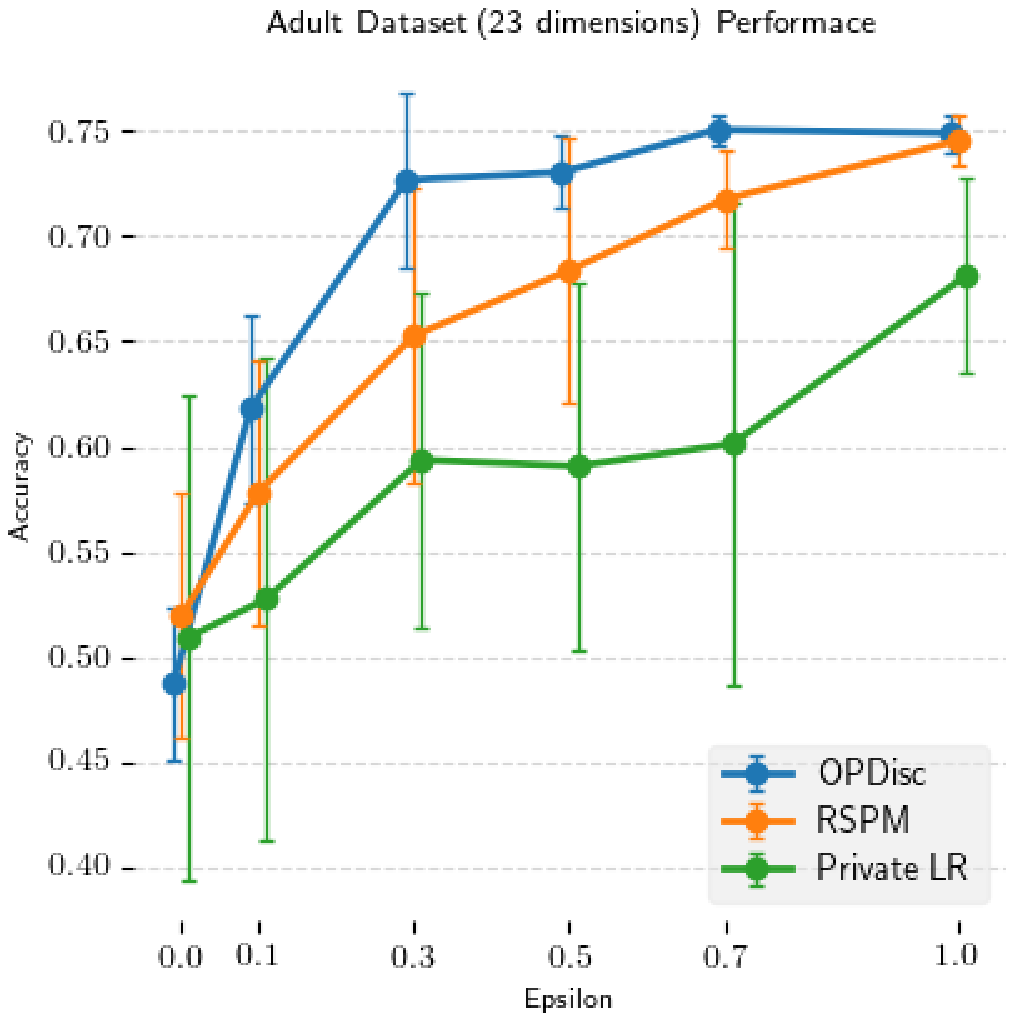}
    \caption{Accuracy versus $\epsilon$.}
    \label{fig:adultAcc}
  \end{subfigure}
  \begin{subfigure}[b]{0.46\textwidth}
    \includegraphics[width=\textwidth]{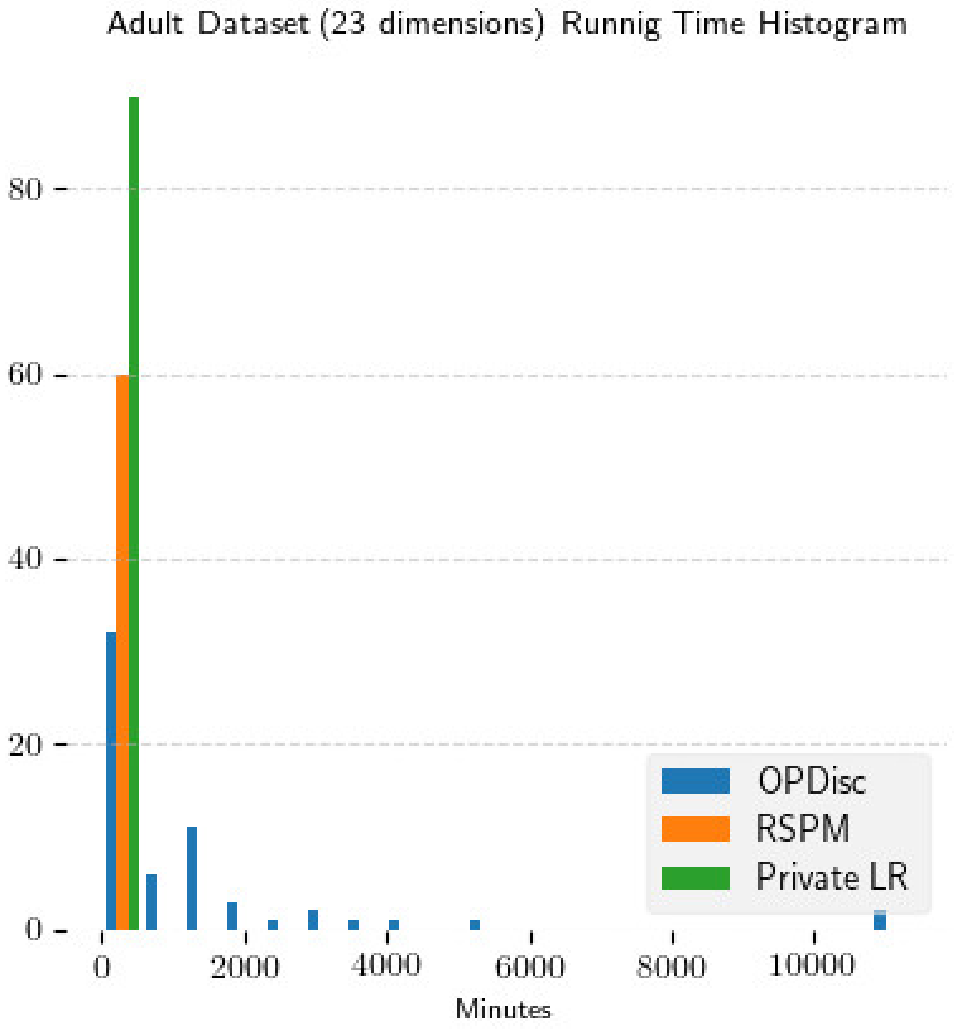}
    \caption{Distribution of run time.}
    \label{fig:adultTime}
  \end{subfigure}
  \caption{Accuracy and runtime evaluation of $\ObjDisc$, RSPM, and
    Private Logistic Regression (LR) on the Adult data set with size
    $n = 15682$ and $d = 23$ features. The value of $\delta = 1/n^2$
    for all methods in all runs.}
\end{figure}

%
\paragraph{Data description and pre-processing}
We use the Adult dataset~\cite{Lichman:2013}, a common benchmark
dataset derived from Census data. The classification task is to
predict whether an individual earns over 50K per year. The dataset has
$n=48842$ records and 14 features that are a mix of both categorical
and continuous attributes.%
The Adult dataset is unbalanced: only 7841 individuals have the
$\geq 50k$ (positive) label. To arrive at a balanced dataset (so that
constant functions achieve 50\% error), we take all positive
individuals, and an equal number of negative individuals selected at
random, for a total dataset size of $n=15682$. We encode categorical
features with one-hot encodings, which increases the dimensionality of
the dataset. We found it difficult to run our algorithm with more than
30 features, and so we take a subset of 7 features from the Adult
dataset that are represented by $d=23$ real valued features after
one-hot encoding. We chose the subset of features to optimize the
accuracy of our logistic regression baseline.

\paragraph{Baseline: private logistic regression (LR).}
We use as our baseline private logistic regression which optimizes
over the space of continuous halfspaces with the goal of minimizing
the logistic loss function, given by
$l_i(w) = \log\left(1 + \exp(-y \langle w, x_i\rangle)\right)$.  We
implement a differentially private stochastic gradient descent
(privateSGD) algorithm from \cite{BST14, abadi2016deep}, keeping track
of privacy loss using the moment accountant method as implemented in the TensorFlow
Privacy Library. The algorithm involves three parameters: gradient
clip norm, mini-batch size, and learning rate. For each target privacy
parameters $(\eps, \delta)$, we run a grid search to identify the
triplet of parameters that give the highest accuracy. To lower the
variance of the accuracy, we also take average over all the iterates
in the run of privateSGD.

\paragraph{Implementation details for $\ObjDisc$ and RSPM}
For both $\ObjDisc$ and RSPM, we encode each record
$(x_i, y_i) \in \DS$ as a $0/1$ loss function:
$l_i(w) = \mathbbm{1}[y_i \neq \text{sgn}( \inner{x_i}{w})]$. For both
algorithms, we have separation parameter $\tau = 1$ and constrains the
weight vectors to have $\ell_2$ norm bounded by $\sqrt{d}$. In
$\ObjDisc$, each coordinate $w_j$ can take values in the discrete set
$\{-B, -B+1, \ldots, B-1, B\}$ with $B = \lfloor \sqrt d\rfloor$, and
we constrain the $\|w\|_2$ to be at most $\sqrt{d}$. In RSPM, we
optimize over the set $\{-1, 0, 1\}^d$.  $\ObjDisc$ requires an
approximate projected linear optimization oracle
(Definition~\ref{approxlinpi}) and RSPM requires a linear
optimization oracle (Definition~\ref{approxlin}).  In the appendix, we
show that the optimization problems can be cast as mixed-integer
programs (MIPs), allowing us to implement the oracles via the Gurobi
MIP solver. The Gurobi solver was able to solve each of the integer
programs we passed it. The source code for \ObjDisc~ is available via GitHub
(\url{https://github.com/giusevtr/private_objective_perturbation}).

\paragraph{Empirical evaluation.}
We evaluate our algorithms by their ($0/1$) classification
accuracy. The \cref{fig:adultAcc} plots the
accuracy of \ObjDisc~and our baseline (y-axis) as a function of the
privacy parameter $\epsilon$ (x-axis), averaged over 15 runs. We fix
$\delta = 1/n^2$ for all three algorithms across all runs. The error
bars report the empirical standard deviation. We see that both
\ObjDisc~and RSPM improve dramatically over the logistic regression
baseline. This shows that in small-data settings, it is possible to
improve over the error/privacy tradeoff given by standard
convex-surrogate approaches by appealing to non-convex optimization
heuristics. \ObjDisc~also obtains consistently better error than
RSPM. The algorithm \ObjDisc~also has a significantly lower variance
in its error compared to the other two algorithms. The
 \cref{fig:adultTime} gives a histogram of the run-time of our
three methods for our experiment. For both \ObjDisc~and RSPM, the running time is dominated by an integer-program
solver. We see that while our method frequently completes quite
quickly (often even beating our logistic regression baseline!), it has
high variance, and occasionally requires a long time to run. However, we were always able to solve the necessary optimization problem, eventually.

\bibliographystyle{alpha}
\bibliography{arxiv}
\newpage

\appendix

\section{Definitions}
\begin{definition}[\cite{goldman1993exact,oracle16}]\label{def:separator}
A set $U \subseteq \lossSet$ is a \emph{separator set} for a parameter space $\cW$ if for every pair of distinct parameters $w, w' \in \cW$, there is an $l \in U$ such that:
$$l(w) \neq l(w')$$
If $|U| = m$, then we say that $\cW$ has a separator set of size $m$.
\end{definition}

\begin{algorithm}[H]
\label{algorithm:rspm}
\SetAlgoLined
\textbf{Given: } A separator set $U = \{ e_1, \ldots, e_m \}$ for class $\discspace$ and optimization oracle for $\discspace^*$\;
\KwIn{$\DS = \{ l_i \}_{i\in[n]}$}
$n \leftarrow |\DS|$ \;
$\sigma \leftarrow  \frac{7  \sqrt{m \ln{1/\delta}}}{\epsilon}$ \;
Draw i.i.d random vector $\eta \sim \mathcal{N} \big(0, \sigma^2 \big)^{d+1}$\;
Construct a weighted dataset $WD$ of size $n + m$ as follows:
$$WD(\DS, \eta) = \{ (l_i, 1) : l_i \in \DS\} \cup \{(e_i, \eta_i) :e_i \in U \} 	$$
$$w \in \argmin_{w^* \in \discspace} \sum_{(l_i, p_i) \in WD} p_i l_i(w)$$

\KwOut{$w$}
\caption{Gaussian Report Separator perturbed Minimum \cite{neel2018use}}
\end{algorithm}

\begin{definition}
A weighted optimization oracle for a class $\cW$ is a function $\oracle: (\lossSet \times \mathbb{R})^* \rightarrow \cW$ that takes as input a weighted dataset $WD \in (\mathcal{L}\times \mathbb{R})^*$ and outputs $w \in \cW$ such that
$$w \in \argmin_{w^* \in \cW} \sum_{(l_i, p_i) \in WD} p_i l_i(w)$$
\end{definition}

\section{Missing Proofs in Section \ref{sec:discrete}}

\textbf{Proof of Lemma~\ref{lemma:uniquew}.}
\begin{proof}
Since $\mathcal{W}_\tau$ is a discrete space, by a union bound it suffices to show that for any pair $w \neq w' \in \mathcal{W}_\tau$, $\pr{L(\DS, w)-\langle \eta, \pi(w) \rangle = L(\DS, w')-\langle \eta, \pi(w') \rangle} = 0$. Since $w \neq w'$, they must differ in at least one coordinate $i$. Condition on the realization of all of the coordinates of $\eta$ but the $i^{th}$, $\eta_{-i}.$ Then $L(\DS, w)-\langle \eta, \pi(w) \rangle = L(\DS, w')-\langle \eta, \pi(w') \rangle$, only if
$$\eta_i = \frac{L(\DS, w')-L(\DS, w) + \sum_{j \neq i}\eta_j (w_j-w_j)}{(w_i'-w_i)}$$
The expression on the righthand side is well-defined since $w_i \neq w_i'$. But then $\eta_i \sim \mathcal{N}(0, \sigma^2)$ even after conditioning on $\eta_{-i}$, and so its probability of taking any fixed value is $0$. This proves the claim.
\end{proof}

\textbf{Proof of Lemma~\ref{lem:gauss_ratio}.}
\begin{proof}
Fix any $r\in \mathbb{R}^{d+1}$, $w\in \discspace$,  and let $v = \noisemap{w}(r) - r = \frac{2}{\gamma}GD^2\pi(\hat{w})$. Note that $\|v\|_2 = \frac{2}{\tau} G D^2$. Fix an orthonormal basis of $\mathbb{R}^{d+1}$, where the first basis vector $b_1$ is parallel to $v$. Let $r^{[1]}$ be the projection of $r$ onto the direction of $b_1$. Then by Lemma $17$ in \cite{neel2018use}:
\begin{equation}\label{equation:pdfG}
\pdf(r) \leq \exp{\bigg( \frac{1}{2\sigma^2}\Big( \|v\|_2^2 + 2\|v\|_2 \|r^{[1]}\|_2 \Big)   \bigg)}\pdf(\noisemap{w}(r))
\end{equation}
The ratio $\pdf(r)/\pdf(\noisemap{w}(r))$ is bounded by $\exp{(\epsilon)}$  in the event that $\|r^{[1]}\|_2 < 2\sigma^2 \epsilon/\|v\|_2 - \|v\|_2/2$. $||r^{[1]}||_2 \sim |\lambda|$, where $\lambda \sim \mathcal{N}(0, \sigma^2)$, and so using a tail bound for the $\chi^2$ random variable, $\|r^{[1]}\|_2 < 2\sigma^2 \epsilon/\|v\|_2 - \|v\|_2/2$ with probability $1-\delta$ so long as $\sigma = \frac{c \|v\|_2 \sqrt{\ln{(1/\delta)}}}{2\epsilon}$ for $c\geq 3.5$.
Since we have that $\|v\|_2=\frac{2}{\tau} G D^2$, for us it suffices to set $\sigma = \frac{7 G D^2 \sqrt{\ln{(1/\delta)}}}{2\epsilon \tau}$. Let $\Lambda = \sigma^2 \epsilon/\|v\|_2 - \|v\|_2/2$ and define the set $C = \{\eta : \|\eta^{[1]}\|_2 > \Lambda\}$. Then since $\pr{\eta \in C} < \delta$, we are done.
\end{proof}

\textbf{Proof of Theorem~\ref{dproof}.}
\begin{proof}  Write $\wopt = \oracle_\pi(\DS, \eta)$. We first want to show that there exists a mapping $\noisemap{\wopt}:\mathbb{R}^{d+1}\rightarrow\mathbb{R}^{d+1}$ such that $\wopt$ is the parameter vector output on any neighboring dataset $\DS'$ when the noise vector is realized as $\noisemap{\wopt}(\eta)$: that is, $\wopt = \oracle_\pi(\DS',\noisemap{\wopt}(\eta))$. Let $S\subset \discspace$ be a subset of discrete parameters. If we can show that $\pr{S} \approx \pr{\noisemap{\wopt}(S)}$, then the probability of outputting any particular $w$ on input $\DS$ should be close to the corresponding probability, on input $\DS'$ as desired. Denote  the set of of noise vectors that induce output $w$ on dataset $\DS$ by $\noiseset{\DS, w} =\{\eta : \oracle_\pi(\DS, \eta)=w  \}$.
Define our mapping:
$$\noisemap{\wopt}(\eta) = \eta + \frac{2}{\tau } G D^2 \bproj{\wopt}  $$
We now use the 3 key Lemmas to finish the privacy proof.
Putting it all together:
%
%
%
\begin{align*}
\pr{\oracle_\pi(\DS,\eta) \in S} &= \pr{\eta \in \bigcup_{\wopt} \noiseset{\DS, \wopt}}\\
&=\int_{\mathbb{R}^{d+1}} \pdf(\eta) \mathbbm{1}\{ \eta \in \bigcup_{\wopt} \noiseset{\DS, \wopt}\}d\eta\\
%
&=\int_{(\mathbb{R}^{d+1}\setminus B)\setminus C} \pdf(\eta) \mathbbm{1}\{ \eta\in \bigcup_{\wopt} \noiseset{\DS, \wopt}\}d\eta + \int_{C} \pdf(\eta) \mathbbm{1}\{ \eta\in\bigcup_{\wopt} \noiseset{\DS, \wopt}\}d\eta\\
&\leq\int_{(\mathbb{R}^{d+1}\setminus C)\setminus B} \pdf(\eta) \mathbbm{1}\{ \eta \in\bigcup_{\wopt} \noiseset{\DS, \wopt}\}d\eta + \delta \quad (\text{Lemma~\ref{lemma:uniquew}, $\pr{\eta \in C} < \delta$}) \\
&=\sum_{\wopt \in S} \int_{\mathbb{R}^{d+1}\setminus (C \cup B)} \pdf(\eta) \mathbbm{1}\{ \eta \in \noiseset{\DS, \wopt}\}d\eta + \delta  \\
&\leq \sum_{\wopt \in S}  \int_{\mathbb{R}^{d+1}\setminus (C \cup B)} \pdf(\eta) \mathbbm{1}\{ \noisemap{\wopt}(\eta) \in \noiseset{\DS', \wopt}\}d\eta + \delta \quad (\text{Lemma~\ref{lemma:indicator}}) \\
&\leq \sum_{\wopt \in S} \int_{\mathbb{R}^{d+1}\setminus (C \cup B)} \exp(\epsilon)\pdf(\noisemap{\wopt}(\eta)) \mathbbm{1}\{ \noisemap{\wopt}(\eta) \in \noiseset{\DS', \wopt}\}d\eta + \delta \quad (\text{bounded ratio})\\
&= \sum_{\wopt \in S}  \int_{\mathbb{R}^{d+1}\setminus (\noisemap{\wopt}(C) \cup \noisemap{\wopt}(B))}  \exp(\epsilon)\pdf(\eta)\mathbbm{1}\{\eta \in \noiseset{\DS', \wopt}\} \bigg|\frac{\partial \noisemap{\wopt}}{\partial \eta} \bigg| d\eta  \quad (\eta \rightarrow \noisemap{\wopt}(\eta)) \\
&\leq \exp(\epsilon) \sum_{\wopt \in S}  \int_{\mathbb{R}^{d+1}}  \pdf(\eta)\mathbbm{1}\{\eta \in \noiseset{\DS', \wopt}\} d\eta + \delta \quad  \\
&= \exp(\epsilon) \pr{\eta \in \bigcup_{\wopt} \noiseset{\DS', \wopt}} \\
&= \exp(\epsilon) \pr{\oracle_\pi(\DS',\eta) \in S} + \delta
\end{align*}
This completes the proof.
\end{proof} 
%
%
%
\section{Missing Proofs in Section~\ref{sec:sampling}}
\textbf{Proof of Lemma~\ref{lemma:close_to_mean_lemma}}
\begin{proof}
If we denote $w(\sigma) = \approxoracle(\DS, \sigma)$ as the output of an approximate oracle on dataset $\DS$ induced by a realization of the noise vector $\sigma$,  then $w(\sigma^1), \ldots w(\sigma^m)$ are $m$ independent random variables with $-D\leq w(\sigma^i)_j\leq D$ for all $i$ and for each coordinate $j\leq d$.

For any index coordinate $j$, let $X_i = (w(\sigma^i)_j + D )/ 2D$, $S=\frac{1}{m}\sum_i^m X_i$ and $\mu_S = \mathbb{E}[S]$. Since $0\leq X_i \leq 1$, by Chernoff bound we have
\begin{equation*}\begin{split}
&\pr{S> \mu_S + \gamma} < e^{-2m\gamma^2}\\
&\pr{\frac{1}{m}\sum_i^m(w(\sigma^i)_j + D )/ 2D> \mu_S + \gamma} < e^{-2m\gamma^2}\\
&\pr{\frac{1}{m}\sum_i^m w(\sigma^i)_j>  2D\mu_S -D +   2D\gamma} < e^{-2m\gamma^2}\\
&\pr{\cW(\DS, \Sigma)_j >  \mathbb{E}_\sigma[\approxoracle(\DS, \sigma)]_j +   2D\gamma} < e^{-2m\gamma^2}
\end{split}\end{equation*}
Plugging in the value of  $m=\frac{-\ln{(\delta/(2d))}}{2\gamma^2}$ we get:
\begin{equation*}\begin{split}
&\pr{\cW(\DS, \Sigma)_j-\mathbb{E}_\sigma[\approxoracle(\DS, \sigma)]_j>  2D\gamma} < \delta/(2d) \\
&\pr{\cW(\DS, \Sigma)_j-\mathbb{E}_\sigma[\approxoracle(\DS, \sigma)]_j<  -2D\gamma} < {\delta}/(2d)
\end{split}\end{equation*}
Thus, by union bound
\[
\pr{\|\cW(\DS, \Sigma) - \mathbb{E}[\approxoracle(\DS, \sigma)]\|_1 > 2D\gamma} \leq \sum_{j=1}^d \pr{\big|\cW(\DS,\Sigma)_j - \mathbb{E}[\approxoracle(\DS, \sigma)_j] \big| > 2D\gamma} < \sum_{j=1}^d\delta/(2d) = \delta/2
\]
\end{proof}
\section{Experiments Details}

\subsection{Implentation Details}
The implementation is written in Python and uses Gurobi as a solver.
We run the experiments on a server machine with an 8-core AMD
processor and 192 GB of RAM.

\subsection{Mixed Integer Programs for $\ObjDisc$ and RSPM}
\label{sec:mip}
We use a mixed integer programs (MIP) to encode the optimization
problems of $\ObjDisc$ and RSPM over the space of $d$-dimentional
discrete halfspaces. The input to our algorithm is a dataset
$\{(x_i, y_i)\}^n$ where $x_i \in \mathbb{R}^d$, $y\in \{-1,1\}$ and a
noise vector $\eta \in \mathbb{R}^{d+1}$. The discretization parameter
is $\tau$ and $D$ is the $\ell_2$-norm bound of $W$.

\begin{figure}[h]\label{algorithm:MIP}
\begin{equation}
\begin{aligned}
& \underset{w\in W}{\text{min}}
& & \sum_{i=1}^n e_i - \sum_{i=1}^d \eta_i w_i / D -  \eta_{d+1}\lambda / D  \\
& \text{s.t.} & &   y_i \sum_{j=1}^d w_j x_j + c e_i> 0 \quad \forall i \in [n]\\
& & & \lambda^2 + \|w\|_2^2 \leq D^2 \\
& & & e_i \in \{0,1\} \quad \forall i \in [n] \\
& & & w_j \in \tau \mathbb{Z} \quad \forall j \in [d]
\end{aligned}
\end{equation}
\caption{MIP oracle used by $\ObjDisc$. The MIP consist of $n$
  integral constraints, $d$ linear and $1$ quadratic constraint.}
\end{figure}

In $\ObjDisc$, the objective we want to minimize is
$\loss(\DS,w) -\inner{\eta}{\bproj{w}}$ which we can rewrite as
\begin{equation}
\label{eq:MIPobjective}
\loss(\DS,w) - \sum_{i=1}^d \eta_i w_i / D - \eta_{d+1}\sqrt{D^2-\|w\|_2^2}/D
\end{equation}
The loss term $\loss(\DS,w)$ in the objective is encoded as a sum of
$n$ binary variables $e_i \in \{0,1\}$, such that if $e_i = 0$ only
then the constraint $y_i \inner{w}{x_i}>0$ must be satisfied. Thus,
the sum $\sum_{i=0}^n e_i$ is equal to the number of misclassified
samples. For each $i \in [n]$, we enconde the constraint corresponding
to $e_i$ in our MIP by the inequality $y_i \inner{w}{x_i}+ c e_i> 0$
where $c$ is a large enough constant with
$c> \max_{x,w} \|x\|_2 \|w\|_2$. The third term in the objective
function \ref{eq:MIPobjective} is non-linear but we can express it as
linear term in the objective by introducing the slack variable
$\lambda$. Then, in order to force the condition that
$\lambda=\sqrt{D^2-\|w\|_2^2}$ we add the quadratic constraint
$\lambda^2 + \|w\|_2^2 \leq D^2$.

\begin{figure}[h]\label{algorithm:MIPRSPM}
\begin{equation}
\begin{aligned}
& \underset{w\in W}{\text{min}}
& & \sum_{i=1}^n e_i -\sum_{i=1}^d \eta_i w_i  \\
& \text{s.t.} & &   y_i \sum_{j=1}^d w_j x_j + c e_i> 0 \quad \forall i \in [n]\\
& & & e_i \in \{0,1\} \quad \forall i \in [n] \\
& & &  -1 \leq w_j  \leq 1 \quad \forall j \in [d] \\
& & & w_j \in \tau \mathbb{Z} \quad \forall j \in [d]
\end{aligned}
\end{equation}
\caption{MIP oracle used by RSPM. The MIP consist of $n$ integral constraints, and $d$ linear constraint.}
\end{figure}
In RSPM, we are simply optimizing the 0-1 loss over the augmented data
set, including the input data set as well as the weighted examples
from the separator set.


\end{document}